\documentclass{ecai}
\usepackage{graphicx}
\usepackage{latexsym}
\usepackage{dirtytalk}
\usepackage{amsmath}
\usepackage{amsthm}
\usepackage[switch]{lineno}
\usepackage{dingbat}
\usepackage{algorithm}
\usepackage{algorithmic}
\usepackage{booktabs}
\usepackage{enumitem}
\usepackage{enumitem}

\newtheorem{theorem}{Theorem}
\newtheorem{lemma}{Lemma}

\ecaisubmission   

\begin{document}

\begin{frontmatter}

\title{Sample-Level Weighting for Multi-Task Learning with Auxiliary Tasks}
\author[A]{\fnms{Emilie}~\snm{Grégoire}}
\author[B]{\fnms{Hafeez}~\snm{Chaudhary}}
\author[A]{\fnms{Sam}~\snm{Verboven}} 

\address[A]{Vrije Universiteit Brussel}
\address[B]{Royal Military Academy}
\begin{abstract}
Multi-task learning (MTL) can improve the generalization performance of neural networks by sharing representations with related tasks. Nonetheless, MTL can also degrade performance through harmful interference between tasks. Recent work has pursued task-specific loss weighting as a solution for this interference. However, existing algorithms treat tasks as atomic, lacking the ability to explicitly separate harmful and helpful signals beyond the task level.
To this end, we propose SLGrad, a sample-level weighting algorithm for multi-task learning with auxiliary tasks. Through sample-specific task weights, SLGrad reshapes the task distributions during training to eliminate harmful auxiliary signals and augment useful task signals. Substantial generalization performance gains are observed on (semi-) synthetic datasets and common supervised multi-task problems.
\end{abstract}
\end{frontmatter}

\section{Introduction}
Joint optimization with auxiliary tasks can lead to a better generalizable solution for the main task of interest while requiring less data. Recently, multi-task learning (MTL) has seen various successful applications in Deep Learning (DL), e.g., in computer vision \cite{c7,c8}, reinforcement learning \cite{c9}, and recommender systems \cite{c11}. \\ \indent Nonetheless, the optimization of MTL networks remains challenging in practice. During training, negative interference between tasks can prevent the network from reaching a good optimum. Negative inter-task interference mainly occurs for two reasons: (i) task updates providing harmful gradient directions or (ii) gradient norms being imbalanced \cite{c16,c17,c14,c15}. Existing approaches prevent negative gradient interference \textit{at the task level} by finding appropriate task weights in the multi-task loss function \cite{C18,c27} or by aligning the different task gradients \cite{c13,c39}. 
\begin{figure}[h!]
\centering
\includegraphics[width=\columnwidth]{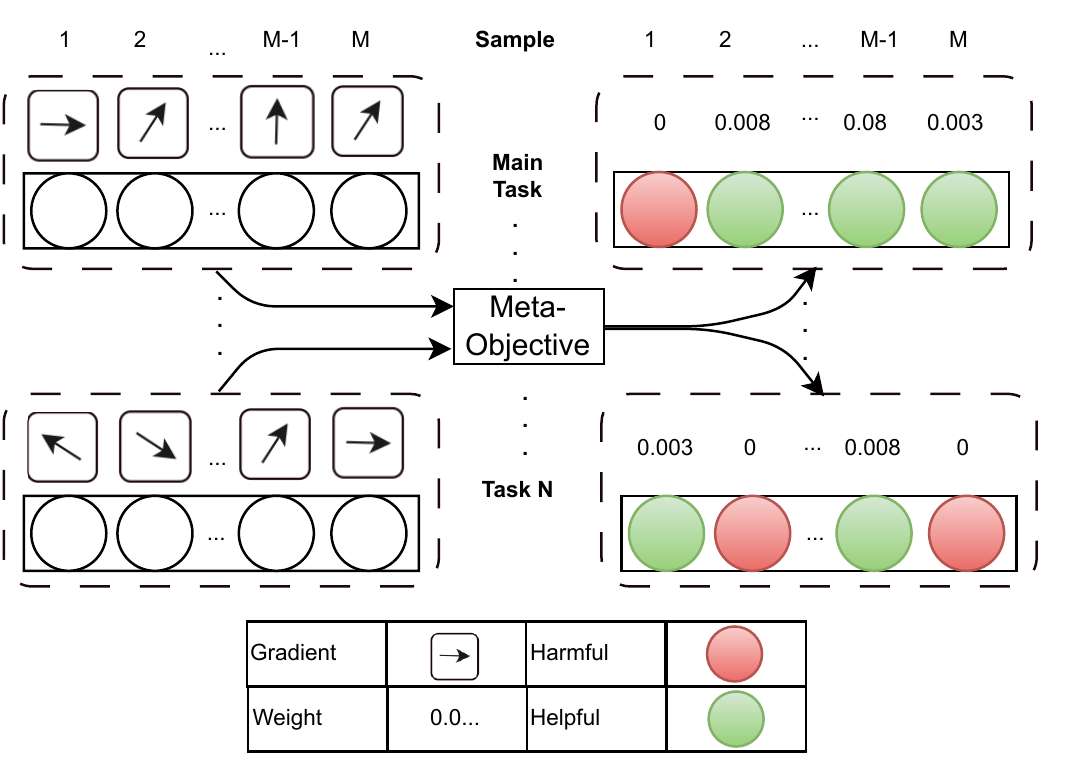}
    \caption{Pictorial representation of the SLGrad task-sample weighting algorithm. SLGrad combines the ability to focus on helpful samples with the ability to prioritize helpful auxiliary tasks for learning a task of interest.}
    \label{drawfig}
\end{figure} 
\begin{table*}[h!]
\begin{center}
{\caption{SLGrad innovations as compared to baseline dynamic task weighting algorithms. The checkmarks indicate which algorithms contain specific innovations discussed in the literature review.}\label{innovation}}
\begin{tabular}{lcccccc}
\hline
\rule{0pt}{12pt}
\rule{0pt}{12pt}
Algorithm&Auxiliary Learning&Highly-Dynamic &Sample-Level&Meta-Objective Learning&Look-Ahead&Paper \\
\\
\hline
\\[-6pt]
OL-AUX &\checkmark&&&&\checkmark& \cite{c9} \\
PCGrad &  & \checkmark &  &   & & \cite{c13}\\
CAGrad &  & \checkmark  &  &  & & \cite{c30}\\
CosSim & \checkmark & \checkmark&        &  & & \cite{c15}\\
GradNorm & & \checkmark& & & &\cite{c2}\\
Random & &\checkmark & &  & & \cite{c38}\\
\textbf{SLGrad} & \checkmark & \checkmark & \checkmark & \checkmark  & \checkmark & Ours  \\
\hline
\\[-6pt]
\end{tabular}
\end{center}
\end{table*}
At the same time, it is widely recognized in machine learning literature that some samples are more useful than others, with some even harming training, e.g., noisy data \cite{c32,c33,c35,c36}.\\ \indent In this paper, we abandon the widespread assertion of MTL literature of a task as an atomic unit. We disentangle the composite task signal into the sum of its parts on a sample-level. Instead of focusing on task-level weighting, this paper adopts a novel unified perspective that jointly optimizes task- and sample-level weights.  \\ \indent To fill the corresponding research gap, we propose SLGrad, a sample-level dynamic task-weighting algorithm for learning with auxiliary tasks. Different from existing task weighting algorithms, SLGrad evaluates the usefulness of samples from each task distribution with respect to the generalization performance of the main task. As such, the proposed algorithm is able to systematically avoid negative interference between different task distributions while providing meaningful sample weighting. SLGrad thus joins the benefits of an auxiliary learning approach with those of sample weighting. A pictorial representation of SLGrad is provided in Figure~\ref{drawfig}. Experiments on synthetic, semi-synthetic, and real data show that SLGrad is flexible to adjust to the variability in task and sample signals, providing state-of-the-art performance. 

The contributions of this paper are summarized below:
\begin{itemize}[topsep=2pt]
   \item We bridge the existing gap between sample weighting in single-task literature and dynamic task weighting in multi-task learning.
  \item We propose a new dynamic task weighting algorithm for MTL with auxiliary tasks: SLGrad.
   \item SLGrad's sample and task weight distributions are exhaustively studied.
    \item Extensive benchmark experiments on (semi)-synthetic and real-world datasets are performed.
\end{itemize}
\section{Related Literature}
In MTL literature, one either cares about optimizing all the tasks or only a single task of interest. This paper is positioned in the latter setting, which we will refer to as auxiliary learning. Here, auxiliary tasks only matter insofar as they help in improving learning the task of interest, denoted as the main task. This work draws from recent developments in dynamic task weighting, sample (i.e., importance) weighting and meta-learning literature. As such, we review these relevant strands of the literature and summarize them in Table \ref{innovation}. 
\subsection{Multi-task Learning with Auxiliary Tasks}
Auxiliary MTL has been successfully applied to many complex, high-dimensional domains, e.g., natural language processing \cite{c20}, computer vision \cite{c21,c23}, and reinforcement learning \cite{c9}. A key challenge for MTL, and by extension auxiliary learning setting is finding the optimal balance between each task in training \cite{c3,C18}. 

\subsection{Dynamic Task Weighting}
To address the challenge of balancing task contributions, dynamic task weighting algorithms have been proposed. In deep MTL, a neural network is typically trained using a linearly weighted combination of individual task losses, i.e., the composite loss. Setting a single uniform weight for each task loss in the composite loss throughout the training of the network has been shown to lead to poor performance \cite{c6}. Furthermore, a particular task may help learn a useful base representation initially but interfere with fine-tuning later on. Consequently, task weights are often designed to be dynamic, i.e., change during training, or highly dynamic, i.e., set independently of the previous update \cite{hyda, c28}. Various highly dynamic task weighting strategies exist, each adapting the loss weights or the gradients of the different tasks. Strategies include weighting based on inverse validation loss \cite{c28}, geometric loss \cite{c29}, and uncertainty \cite{c27}. For a complete overview, we refer to \cite{c3} and \cite{C18}.
\paragraph{Negative Task Interference.} A task is said to be negatively interfering when its learning signal negatively impacts learning the main task. Reasons could be that the gradient magnitudes are imbalanced \cite{c2,c17,c16} or that the gradient directions of the updates simply do not align with a direction increasing generalization performance \cite{c15}. To avoid task gradients with low gradient magnitudes being overshadowed by tasks with large gradients, gradient magnitudes can be normalized throughout training with respect to one main or several tasks \cite{c2,c17}. Many different approaches to deal with harmful or conflicting gradient directions exist, e.g., dropping contributions with negative cosine similarity with the average direction \cite{c14}, projecting gradients to either the normal plane of the other task \cite{c13} or an optimized vector \cite{c30,c39}. All of these strategies, however, treat tasks as atomic units, with the exception of \cite{robmtl}. An important research question is thus: \say{how to account for differences in task usefulness beyond the task level in MTL?}, as we know from the literature that some samples matter more than others. 

\subsection{Some Samples Matter More than Others}
In single-task learning, it has been well-established that some samples matter more than others for learning well-generalizable solutions \cite{c40}. Some samples can be removed from the dataset at the beginning or during training without any loss of generalization performance \cite{c33}. Others, through data set biases such as noisy or class-imbalanced data, can explicitly harm performance \cite{imp}. Introducing sample-level weighting in the loss function is a common strategy to mitigate these biases.
\paragraph{Sample Weighting.} Sample-level weighting for deep learning is used across many domains such as causal inference and off-policy reinforcement learning. In causal inference, inverse propensity weighting  \cite{c42} balances the distributions of the control and treatment groups in observational data. Similarly, sample weighting can reduce training set bias caused by class imbalance \cite{c801,c802} or label noise \cite{c800,c35,c809}. Finally, sample weights can take into account prioritization based on domain knowledge e.g., in cost-sensitive learning \cite{c803,c804}.  In effect, across these domains, weighting samples of the original source distribution leads an algorithm to learn from a semi-synthetic target distribution favourable to the task of interest. This mechanic is explicitly exploited in the importance weighting literature \cite{impweight}.

\subsection{Generalization as a Meta-Objective}
Meta-learning \cite{metalearn} uses the learning performance of machine learning approaches to improve the learning of the same or other tasks \cite{MLearning}. Recent work has instantiated hold-out (i.e., validation set) performance on the task of interest as an optimization-based meta-learning objective \cite{c35}. The incorporation of hold-out performance in the outer training loop is commonly used in the field of meta-learning to prioritize solutions that generalize better \cite{metasurv,c810}. 
\section{SLGrad: Sample-Level Weighting for MTL With Auxiliary Tasks}
To unify sample-level weighting with MTL, we propose SLGrad, a sample-level weighting algorithm for auxiliary learning that dynamically assigns higher weights to samples that positively contribute to the generalization performance of the main task, while discarding negatively contributing samples.  
This section is organized as follows: first, the problem setting is defined in subsection \ref{formulation}. Next,  the main components of SLGrad are introduced and motivated in subsection \ref{solutionn}. Third, the algorithm and its implementation details are formalized in subsection \ref{algorithmm}. Finally, the main theoretical properties of the algorithm are presented in subsection \ref{theo}.
\subsection{Problem Formulation \label{formulation}}
The objective is to optimize the generalization performance of a deep neural network on a main task $\mathcal{T}_{m}$, measured by $\mathcal{M}_{m}$, through joint training with a set of auxiliary tasks $\mathcal{T}_{a}$. \\
\indent The standard multi-task loss $\mathcal{L}^{(t)}_{MTL}$ consists of a linear combination of task losses $\mathcal{L}_{i}$, 
\begin{equation}
    \mathcal{L}^{(t)}_{MTL}=\sum_{i=1}^{N_{T}}w_{i}\mathcal{L}_{i}(\theta_{s}^{(t)}, \theta_{i}^{(t)}),\\
    \label{stdmtl}
\end{equation}
where $N_{T}$ denotes to the number of different tasks, $\theta_{s}^{(t)}$ the model parameters of the shared layers, and $\theta_{i}^{(t)}$ the parameters of the task-$i$-specific layers at training step $t$. Finally, the coefficients $w_{i}$ correspond to the task-specific weights. \\
\indent The model parameters $\theta^{(t)}$, both shared and task-specific, are updated based on a mini-batch of data, using standard gradient descent algorithms on $\mathcal{L}^{(t)}_{MTL}$:
\begin{equation}
    \theta^{(t+1)} = \theta^{(t)} -\eta \nabla_{\theta^{(t)} }\mathcal{L}^{(t)}_{MTL}, 
    \label{gradd}
\end{equation}
where $\eta$ is the learning rate. 

This standard setup in multi-task learning only facilitates weighting at the task-level.
\subsection{Proposed Solution: SLGrad \label{solutionn}}
\indent In order to separate harmful from useful signals \textit{beyond the task level}, the standard MTL problem must be extended to facilitate the weighting of individual task-sample combinations. As such, we extend the standard multi-task loss in equation (\ref{stdmtl}) to:
\begin{equation}
    \mathcal{L}^{(t)}_{SLGrad}=\sum_{i=1}^{N_{T}}\sum_{j=1}^{N_{B}}w^{(t)}_{ij}\mathcal{L}_{ij}(y_{ij},\hat{y}_{ij},\theta_{s}^{(t)}, \theta_{i}^{(t)}),
    \label{SLGradloss}
\end{equation}
where $N_{B}$ represents the size of the mini-batch. In this paper, the goal is to find the task-sample-specific weights $w^{(t)}_{ij}$ such that the generalization metric $\mathcal{M}_{m}^{val}$ is minimized. As such, we propose SLGrad, an algorithm that jointly solves the task and sample weighting problems. We thus tackle the MTL weighting problem at a more granular level compared to current solutions. In effect, such a solution jointly tackles the multi-task weighting problem from MTL literature, and the sample weighting problem from general machine learning literature. \\
\indent Specifically, SLGrad dynamically optimizes task-sample-specific importance weights at each step of the training according to a simple and intuitive update rule. 
The sample-level weights $w^{(t)}_{ij}$ are computed as
\begin{equation}
 \Tilde{w}^{(t)}_{ij}=\nabla _{\theta^{(t)}}\mathcal{L}_{ij}^{train} \cdot \nabla_{\theta^{(t)}}\mathcal{M}^{val}_{m},
 \label{cossim} 
\end{equation} 
where $\mathcal{L}_{ij}^{train}$ represents the task-sample specific training loss and $\mathcal{M}^{val}_{m}$ represents the generalization metric.
Next, the sample-level weights are normalized according to 
\begin{equation}
    w^{(t)}_{ij}=\frac{\max(\Tilde{w}_{ij}^{(t)},0) }{\sum_{i=1}^{N_{T}}\sum_{j=1}^{N_{B}}\max(\Tilde{w}^{(t)}_{ij}, 0)}. 
    \label{weightup}
\end{equation}
In the following section, we explain how the SLGrad weighting rule defined by equations (3) to (5) incorporates the following features: 
\begin{itemize}
    \item Prioritization of the main task
    \item Highly dynamic weighting
    \item Sample-level weighting
    \item Generalization as a meta-objective
    \item Look-ahead updating
\end{itemize}
A feature-wise comparison of SLGrad and competing algorithms is depicted in Table \ref{innovation}.
\paragraph{Prioritization of the main task.}
In contrast to the majority of previous dynamic weighting algorithms, which focus on a multi-objective setting, equation (\ref{cossim}) is tuned to the auxiliary learning set-up, explicitly prioritizing learning the main task $m$. As implied by Theorem \ref{taylortheo}, the decrease in the generalization loss of the main task is proportional to the size of the weights $w_{i,j}$ (defined by equation \ref{cossim}), i.e., samples that help to decrease $\mathcal{M}^{val}_{m}$ receive a higher weight. 
\paragraph{Highly dynamic updates.} 
We speak of highly dynamic weight updates when $w_{ij}^{(t+1)}\perp \!\!\! \perp w_{ij}^{(t)}$. As each mini-batch contains different samples, this is a necessary condition for sample-level weighting. As equation (\ref{cossim}) does not depend on the previous weights, the condition naturally holds for SLGrad. Moreover, previous work has argued that highly dynamic updates can better adjust to variability in task usefulness on a mini-batch basis as well \cite{hyda}. To assure full independency between two parameter updates, SLGrad initializes the sample-level weights to uniform weights before performing the first forward and backward passes used to compute the sample-level gradients.  
\noindent\paragraph{Sample-level weighting.}
Cosine similarity is commonly used to quantify the similarity between gradient directions \cite{C18,c3}. Equation (\ref{cossim}) reflects the cosine similarity between the gradient of a training sample $(x_{ij},y_{ij})$ and the gradient of the main task generalization metric $\mathcal{M}^{val}_{m}$.
Then, using equation (\ref{weightup}), negative interference between samples and the main task is avoided by assigning zero weight to samples whose gradient is misaligned with the main task gradient ($\Tilde{w}_{ij}<0$). Samples that are well aligned with respect to this objective ($\Tilde{w}_{ij}>0$) are weighted proportionally to their similarity with the gradient direction of $\mathcal{M}^{val}_{m}$. 
The denominator in equation (\ref{weightup}) ensures that the weights are normalized, avoiding fluctuations in the total learning rate \cite{c16}. Furthermore, this normalization factor is model-agnostic, meaning that we could replace equation (\ref{weightup}) with any other gradient normalization rule, e.g., \cite{c2,c17}. In summary, the sample weighting mechanism of SLGrad handles both sources of negative task interference: harmful gradient directions and imbalanced gradient magnitudes.
\paragraph{Generalization as a meta-objective.}
Previous algorithms that focus on main task training loss $\mathcal{L}_{ij}^{train}$ \cite{c13,c30} are prone to overfitting. Equation (\ref{cossim}) uses the gradient of the validation (minimization) metric $\nabla_{\theta^{(t)}}\mathcal{M}^{val}_{m}$ as a meta-objective that provides a proxy for generalization performance. In doing so, we draw from meta-learning literature to incorporate an explicit generalization objective. This approach contrasts with algorithms that use gradient directions computed on the main task training loss as an optimization guide. A second benefit of this approach, which we will validate experimentally, is that if this meta-objective is unbiased, the weighting mechanism will also have a debiasing effect on the learned representation.
\paragraph{Look-ahead update.}
Look-ahead approaches locally scout the loss surface to inform optimization \cite{c9,hyda,TAG}. The update rule in equation (\ref{cossim}) represents a computationally efficient version of a look-ahead update \cite{c15,hyda,TAG}, i.e., information from a throw-away update is used in $\mathcal{M}_{m}^{val}(\theta^{(t+1)})$. Note that the extra update to arrive at $\mathcal{M}_{m}^{val}(\theta^{(t+1)})$ is not explicitly calculated in our algorithm; its difference with $\mathcal{M}_{m}^{val}(\theta^{(t)})$ is approximated using Theorem \ref{taylortheo}. Furthermore, we justify the use of the approximation in the supplemental material (section C1). 
\subsection{The Algorithm \label{algorithmm}}
Based on equations (\ref{SLGradloss}) to (\ref{weightup}), we can formulate the SLGrad algorithm:
\begin{enumerate}
    \item Collect sample-task-specific gradients $\nabla_{\theta^{(t)}}\mathcal{L}_{ij}^{train}$ by performing a forward and backward pass on a mini-batch sampled from the training set. The backpropagated total loss is constructed with uniform task weights $w_{ij}^{(t)}=\frac{1}{N_{B}*N_{T}}$ based on the mini-batch size $N_{B}$ and the number of tasks $N_{T}$.
    \item Determine the gradient of the main task metric $\nabla_{\theta^{(t)}}\mathcal{M}^{val}_{m}$ computed by performing a forward and backward pass on a batch of the validation set.
    \item Compute the cosine similarity $\nabla _{\theta^{(t)}}\mathcal{L}_{ij}^{train} \cdot \nabla_{\theta^{(t)}}\mathcal{M}^{val}_{m}$ between all gradients obtained in steps (1) and (2). 
    \item Reweight all samples of each task by following the SLGrad update rule defined in equations (\ref{cossim}) and (\ref{weightup}).
    \item Use the computed weights $w_{ij}^{(t)}$ from step (4) to perform a final backward pass with the resulting total loss $\mathcal{L}^{(t)}_{SLGrad}$ and update the model parameters through gradient descent step $\theta^{(t)} \to \theta^{(t+1)}$ (equation \ref{gradd}).
\end{enumerate}
The full SLGrad procedure is described in Algorithm \ref{alg}. 
\begin{algorithm}[tb]
\caption{SLGrad Algorithm}
\label{alg}
\vspace{5mm}
\textbf{Input}: Deep MTL Network $\mathcal{N_{MTL}}$, $N_{T}$ tasks, Main task metric $\mathcal{M}_{m}^{val}$, Batch size $N_{B}$,  $N_{T}*N_{B}$ task losses $\mathcal{L}_{ij}$, learning rate $\eta$, Number of training steps $N_{s}$, Initial model parameters $\theta^{(0)}$\\
\begin{algorithmic}[1] 
\STATE \textbf{For $t=0$ to $t=N_{s}-1$:}
\STATE \hspace{3mm} Sample batch from training set $\gets (X,Y)_{train}$
\STATE \hspace{3mm} Sample batch from validation set $\gets (X,Y)_{val}$
\STATE \hspace{3mm} Initialize model with model parameters $\theta^{(t)}$ and uniform sample weights $w_{ij}=(N_{T}*N_{B})^{-1}$ \\ 
\STATE \hspace{3mm} \textbf{Compute gradient validation metric} \\
\hspace{10mm} $\nabla_{\theta^{(t)}}\mathcal{M}^{val}_{m} \gets \mathcal{N}_{MTL}[(X,Y)_{val}, \theta^{(t)}]$ \\
\STATE \hspace{3mm} \textbf{Compute sample-level task gradients}\\
\hspace{6mm} \textbf{For $i \in \{0,..., N_{T}-1\}$:}\\
\hspace{9mm} \textbf{For $j \in \{0,..., N_{B}-1\}$:}\\
\hspace{12mm} $\nabla_{\theta^{(t)}}\mathcal{L}^{(t)}_{ij}\gets\mathcal{N}_{MTL}[(x_{ij},y_{ij})_{train}, \theta^{(t)},w_{ij}]$ \\
\STATE \hspace{12mm} \textbf{Determine cosine similarity}\\
\hspace{12mm} $\Tilde{w}^{(t)}_{ij}\gets\nabla_{\theta^{(t)}}\mathcal{M}^{val}_{m}\cdot \nabla _{\theta^{(t)}}\mathcal{L}_{ij}^{(t)}$
\STATE \hspace{3mm}\textbf{Clamp and normalize all weights}\\
\hspace{6mm} $w^{(t)}_{ij}=\frac{\max(\Tilde{w}^{(t)}_{ij}, 0) }{\sum_{i}\sum_{j}\max(\Tilde{w}^{(t)}_{ij},0)}$
\STATE \hspace{3mm} \textbf{Update model parameters with new weights} \\
\hspace{6mm} $\theta^{(t+1)} \gets \theta^{(t)} - \eta \nabla_{\theta^{(t)}}\mathcal{L}_{SLGrad}$
\STATE \textbf{End For}
\vspace{5mm}
\end{algorithmic}
\end{algorithm}
\subsection{Theoretical Analysis \label{theo}}
In this subsection, we theoretically motivate our update rule (equations (\ref{cossim}) and (\ref{weightup})) through two theorems. First, we prove that, under general assumptions, SLGrad ensures the meta-minimization objective $\mathcal{M}^{val}_{m}$ monotonically decreases between every two updates. Second, we show that the cosine similarity between sample level gradients $\nabla _{\theta^{(t)}}\mathcal{L}_{ij}^{train}$ and the gradient of the generalization metric $\nabla_{\theta^{(t)}}\mathcal{M}^{val}_{m}$ is approximately proportional to the decrease in the main task metric $\mathcal{M}^{val}_{m}$.  Finally, we discuss SLGrad's computational efficiency.
\begin{table*}[t]
\begin{center}
 \caption{Comparison of the test set performance (Mean Squared Error) for the toy set-up with different noise levels. SLGrad significantly outperforms the other benchmarks. The hyperparameters are optimized for each algorithm. We compute the mean performance over different initializations.}
 \label{baselinetoy}
\begin{tabular}{lllllllll}
\hline
\rule{0pt}{12pt}
\rule{0pt}{12pt}
\textbf{Noise} & Static & OL-AUX & PCGrad & CAGrad & CosSim & GradNorm & Random & SLGrad\\
\\
\hline
\\[-6pt]
40 $\%$ & 0.86 $\pm$ 0.03 & 0.89 $\pm$ 0.02 & 0.90 $\pm$ 0.02& 0.89 $\pm$ 0.02 &  0.73 $\pm$ 0.38 & 0.88 $\pm$ 0.01 & 0.59 $\pm$ 0.35 &  \textbf{0.03 $\pm$ $3E^{-4}$}\\ 
 70 $\%$ & 1.07 $\pm$ $3E^{-3}$ & 1.27 $\pm$ 0.02 & 1.37 $\pm $ 0.05 & 1.13 $\pm$ 0.10 & 0.94 $\pm$ 0.18 & 1.04 $\pm$ 0.02 & 1.05 $\pm$ $6E^{-4}$ & \textbf{0.04 $\pm$ $9E^{-4}$}\\
\hline
\\[-6pt]
\end{tabular}
\end{center}
\end{table*}
\begin{theorem}
\label{theo1}
Assume $\mathcal{M}^{val}_{m}$ is a differentiable function representing the meta-(minimization)- objective as measured on the main task validation set, $\mathcal{L}_{SLGrad}$ (as defined in equation (\ref{SLGradloss})) is the total training loss and $\mathcal{L}_{ij}(\theta^{(t)})\geq 0$ is the sample-level loss corresponding to sample $i$ of task $j$. If the sample-level weights $w_{ij}$ are computed by the SLGrad update rule (\ref{cossim}) and (\ref{weightup}), and the model parameter vector $\theta^{(t)}$ is updated through gradient descent (\ref{gradd}), then the following statement holds:
\begin{equation}
    \mathcal{M}^{val}_{m}(\theta^{(t+1)}) - \mathcal{M}^{val}_{m}(\theta^{(t)}) \leq 0
    \label{mono}
\end{equation}
$\forall t \in [0, \#training steps]$.
\end{theorem}
\vspace{2mm}
The proof of this Theorem is provided in the Supplemental Material (section A). \\
\indent If the validation metric $\mathcal{M}^{val}_{m}$ in Theorem \ref{theo1} is also bounded by below an important implication is the convergence of the main task validation metric to a minimum. Additionally, assuming $\mathcal{M}^{val}_{m}$ is convex, this will correspond to the global minimum. The equality in equation (\ref{mono}) holds only when an optimum of $\mathcal{L}_{SLGrad}$ or $ \mathcal{M}_{m}^{val}$ is reached.  \\
\indent Next, extending upon the proofs by \cite{hyda,c15} we show that the sample-task level weights resulting from equation (\ref{cossim}) are proportional to the decrease of the main task objective metric $\mathcal{M}^{val}_{m}$. 
\\
\begin{theorem}
Assume $\mathcal{M}^{val}_{m}$ is a differentiable function, representing the meta-objective as measured on the
main task validation set and $\mathcal{L}_{SLGrad}$ the total training loss as defined in equation (\ref{SLGradloss}). Consider $\theta^{(t)}$ as the model parameter vector at time $t$.  If the sample-level weights $w_{ij}$ are computed by the SLGrad update rule (defined by equations (\ref{cossim}) and (\ref{weightup})), and the model parameter vector $\theta^{(t)}$ is updated through gradient descent (equation (\ref{gradd})), then the decrease in the validation metric can be approximated as
\label{taylortheo}
\begin{equation}
     \mathcal{M}_{m}^{val}(\theta^{(t+1)}) - \mathcal{M}_{m}^{val}(\theta^{(t)}) \approx - \eta[\nabla_{\theta^{(t)}}\mathcal{M}^{val}_{m}]^{T}\nabla_{\theta^{(t)}}\mathcal{L}_{SLGrad} ,
     \label{tayl}
\end{equation}
where $\eta$ is the learning rate. 
\end{theorem}
The proof is provided in the Supplemental Material (section A).
\indent Theorem \ref{taylortheo} shows that (i) the higher the sample weight $w_{ij}^{(t)}$ (equation \ref{cossim}), the greater the decrease of the validation metric $\mathcal{M}_{m}(\theta^{(t+1)})$ the sample induces, and (ii) that no explicit computation of $\mathcal{M}_{m}(\theta^{(t+1)})$ is required. As such, no explicit computation of the look-ahead update is needed, resulting in greater computational efficiency for the algorithm. We provide experimental justification for the use of the approximation in the Supplemental Material (section C).
\paragraph{The computational efficiency of SLGrad.}
SLGrad has a total computational overhead, compared to static MTL, with one forward and two backward passes. However, two main efficiency measures are taken: compared to other algorithms that apply a similar validation meta-objective \cite{c35}, we save another forward pass and parameter update by using the approximation in equation (\ref{tayl}). Furthermore, the recently introduced Python library \textit{Opacus} \cite{opacus} facilitates a dramatical speed-up in the calculation of sample-level gradients (i.e., line 6 in Algorithm \ref{alg}).

\section{Simulation Study \label{sim}}
Using a toy set-up, we explore how SLGrad handles (i) the sample-level weighting problem, and (ii) the task-level weighting problem. 

Aside from a comparison of SLGrad to state-of-the-art dynamic weighting algorithms, we investigate:
\begin{itemize}
    \item the sample-level weighting distribution throughout training.  
    \item the influence of sample-level weighting on aggregate task prioritization.
\end{itemize}

\subsection{Set-up}
Following  \cite{hyda,c2}, a multi-task toy regression set-up is generated based on the following function:
\begin{equation}
    y_{out}=f_{i}(x_{in})= \sigma_{i}\tanh [(\textbf{B}+\epsilon_{i})\textbf{x}],
    \label{basis}
\end{equation}
where $x_{in}$ is the ten-dimensional input vector. The common basis $\textbf{B}$ and task-dependent $\epsilon_{i}$ represent constant matrices generated IDD from $\mathcal{N}(0,\sqrt{1})$ and $\mathcal{N}(0,\sqrt{3.5})$ respectively. Additionally, $\sigma_{i}$ is a task-dependent scalar that determines the scale of each task. To ensure the beneficial potential of effective sample-level weighting in this scenario, we corrupt part of the training data by adding noise as follows $y_{out}=y_{out}+\mathcal{N}(0,\sqrt{2})$. Samples from the original distribution are labeled as \textit{clean}, corrupted samples as \textit{noisy}. We use the validation set loss as the meta-objective. The full implementation details, including the model backbone used for all models, are included in supplement B.  
\subsection{Results}
Our results on the toy example show that SLGrad is able to jointly solve the sample- and task-level weighting problems.
\begin{figure}[h!]
    \centering
    \includegraphics[width=\columnwidth]{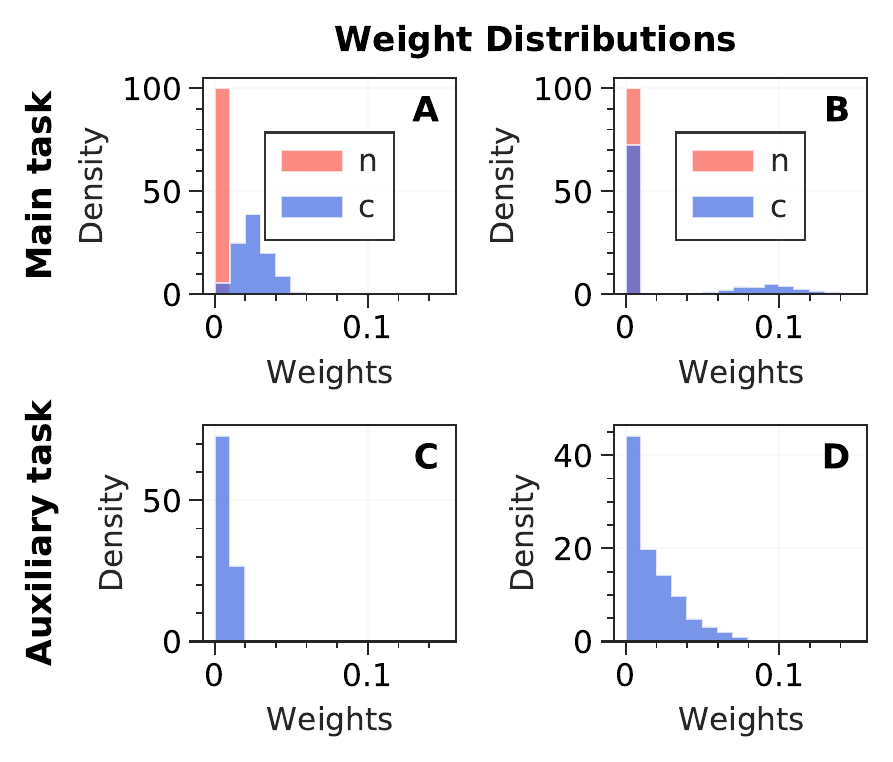}
    \caption{Weight distributions of the main task, for clean and noisy samples (top), and for an auxiliary task (bottom).  Panel A is extracted in epoch 5, B in epoch 30. Panels C and D depict the auxiliary task weight distributions for 40 and 70 percent noise, respectively. SLGrad shifts weight from the full clean distribution to hard examples, and to relatively helpful tasks. We provided the learning curves corresponding to the experiments corresponding to Figure \ref{Fortysss} in the Supplemental Material (section C).}
    \label{Fortysss}
\end{figure}
\paragraph{Sample-level weighting.}
Figure~\ref{Fortysss} shows that during the initial stages of training (panel A), SLGrad systematically cancels out noisy samples by assigning zero weight. At the same time, clean samples are effectively exploited (i.e., positive weights) to learn the objective distribution. As training progresses, fewer clean samples remain useful for improving generalization performance (panel C). The progression from panel A to B illustrates how the algorithm first learns the general distribution and then focuses on the harder examples (cfr. hard example mining), while still ignoring the noisy samples. These dynamics were observed for the main task and the auxiliary tasks alike.
\paragraph{Task-level weighting.}
Going from panel C to D on Figure~\ref{Fortysss}, we observe that SLGrad increasingly assigns positive weights to samples belonging to $\mathcal{T}_{a}$ to learn a good representation for $\mathcal{T}_{m}$. On aggregate, the auxiliary task distribution thus receives more weight as the main task becomes more noisy. This finding appeals to the general intuition that the weight distribution should adapt to a shift of signal across tasks. Finally, we note that SLGrad significantly outperforms the MTL benchmarks (Table~\ref{baselinetoy}) in terms of generalization performance. This performance discrepancy is the direct result of classic MTL algorithms being unable to adapt weights beyond the task level. As a result, the benchmark algorithms cannot effectively differentiate between the two types of samples during model training. 
\section{Real-World Experiments \label{RW}}
In this section, the generalization performance of SLGrad is benchmarked in a series of experiments based on real-world and semi-synthetic data. First, the empirical set-up and baseline algorithms are introduced in subsection \ref{datasets} and subsection \ref{baselines} respectively. Second, the resulting empirical analysis is presented in subsection \ref{results}.
\subsection{Set-up\label{datasets}}
We evaluate the generalization performance and robustness of SLGrad on two real-world datasets commonly used in MTL benchmarks: CIFAR-10 \cite{cifar10} and Multi-MNIST \cite{c1000}. For all experiments, we adapt the LeNet-5 architecture \cite{cLENET} as the backbone for the shared layers. Binary cross-entropy and cross-entropy are used as the task-specific loss functions for CIFAR-10 and Multi-MNIST, respectively. All hyperparameters are optimized through grid search for each algorithm while averaging several initializations. The full configurations for all the experiments are in the supplemental material (section B). For reproducibility purposes, the Python code used to conduct the experiments in this paper will be released. 
\begin{table*}[t]
\begin{center}
{\caption{Comparison with state-of-the-art MTL weighting algorithms on original and label flipped CIFAR-10 set-ups and clean Multi-MNIST. We compare the performance in terms of (Binary) Cross-Entropy for CIFAR-10 and Multi-MNIST. The best results for each experiment are marked in bold. The hyperparameters are optimized for each algorithm. We compute the mean performance over different initializations.}\label{baselineRW}}
\begin{tabular}{l l l l l l  }
\hline
\rule{0pt}{12pt}
\rule{0pt}{12pt}
Model & Clean  & Uniform Flip 40  & Uniform Flip 70 & Background Flip & Clean Multi-MNIST 
\\
\hline
\\[-6pt]
 Static & 0.336 $\pm$ 0.0038 & 0.586 $\pm$ 0.0018 & 0.86 $\pm$ 0.03 & 0.41 $\pm$ 0.09 & 1.95 $\pm$ 0.38 \\
OL-AUX & 0.237 $\pm$ 0.0024 & 0.591 $\pm$ 0.0272 & 0.96 $\pm$ 0.03  & 0.34 $\pm$ 0.03 & 1.16 $\pm$ 0.01  \\
 PC Grad & 0.328 $\pm$ 0.0002 & 0.585 $\pm$ 0.0012
 &  0.84 $\pm$ 0.05  & 0.36 $\pm$ 0.02 & 1.22 $\pm$ 0.03\\
 CA Grad & \textbf{0.235} $\pm$ 0.0031 &  0.569 $\pm$ 0.0038 & 0.94 $\pm$ 2$E^{-3}$ & 0.34 $\pm$ 0.04 & 1.19 $\pm$ 0.02\\\
 CosSim & 0.288 $\pm$ 0.0027 & 0.571 $\pm$ 0.0035 & 0.67 $\pm$ 0.01  & 0.69 $\pm$ 0.14 & 1.38 $\pm$ 0.10\\
 GradNorm & 0.238 $\pm$ 0.0302  & 0.557 $\pm$ 0.0017 & 1.09 $\pm$ 0.02 & 0.33 $\pm$ 0.05 & 1.29 $\pm$ 0.33\\
 Random & 0.251 $\pm$ 0.0012 & 0.572 $\pm$ 0.0027 & 0.96 $\pm$ 0.02 & 0.29 $\pm$ 0.03 & 1.20 $\pm$ 0.03\\
 SLGrad & 0.249 $\pm$ 0.0072 & \textbf{0.288} $\pm$ 0.0263 & \textbf{0.33} $\pm$ 0.04 & \textbf{0.28} $\pm$ 0.02 & \textbf{0.77} $\pm$ 0.03 \\
\hline
\\[-6pt]
\end{tabular}
\end{center}
\end{table*}
\paragraph{Experiments.}
We use the following set-ups for Multi-MNIST, CIFAR-10:
\begin{itemize}
    \item \textit{Clean}: the original dataset without any semi-synthetic interventions. For Multi-MNIST and CIFAR-10, the leftmost digit and the airplanes class are chosen as the respective main tasks.
\end{itemize}
Additionally, the effectiveness of SLGrad is evaluated for the following common semi-synthetic extensions \cite{c35} of CIFAR-10:
\begin{itemize}
    \item \textit{Uniform Label Flips}: samples from all classes are assigned a different label with 40\% and 70\% probability, respectively. 
    \item \textit{Background Label Flips}: all flipped labels (20 percent) are flipped to the same class, i.e., the background class. This type of flip also induces class imbalance as the background class will be likely to dominate. 
\end{itemize}
\begin{figure}[h]
    \includegraphics[width=\columnwidth]{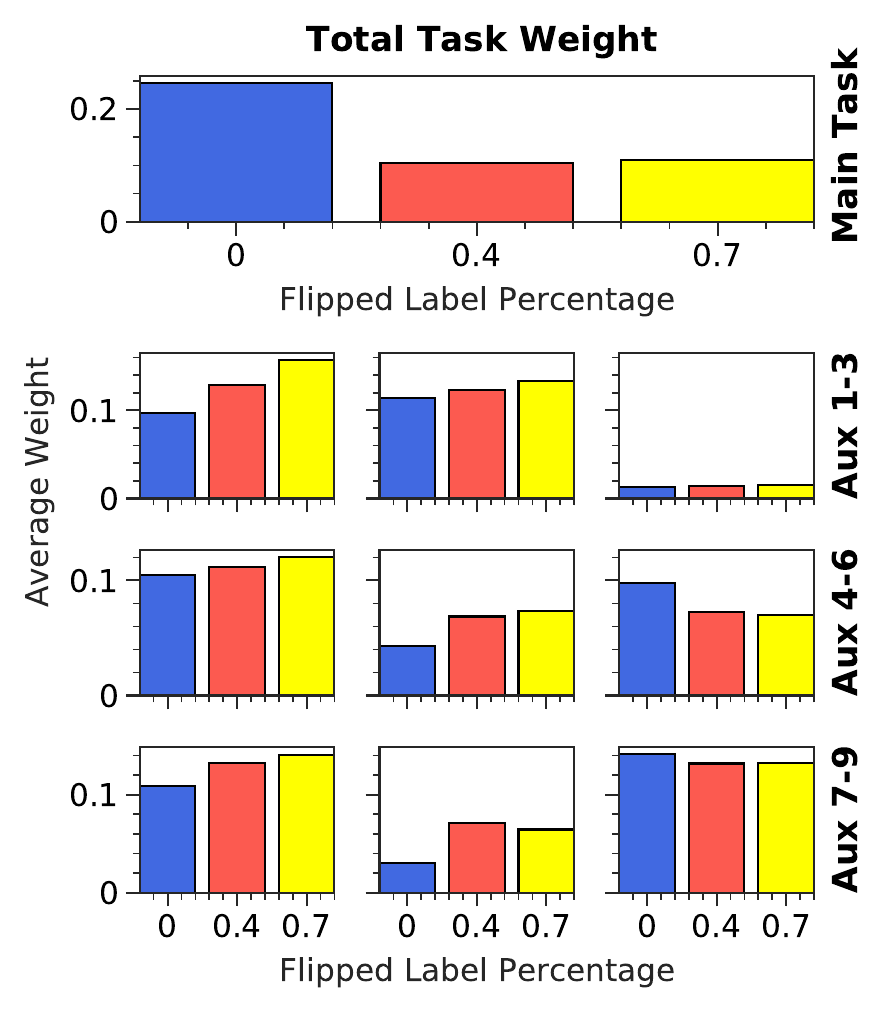}
    \caption{The total average task weights for the main task and the nine auxiliary tasks assigned by SLGrad on CIFAR-10 for 0 (blue), 40 (red) and 70 (yellow) percent of flipped labels. The average task weights are computed by summing over the sample-level weights for all tasks and then taking the time average. While the weights of the main task decrease for an increasing amount of noise, the auxiliary task weights increase.}
    \label{Totalweights}
\end{figure}
More information about the experimental setups is provided in the supplemental material, section B.
\subsection{Baseline Algorithms  \label{baselines}}
We compare SLGrad with seven baseline algorithms: \textit{Static Weighting}: A standard baseline with static and uniformly assigned task weights. \textit{Random Weighting (RW)}: A simple weighting method where a model is trained with randomly assigned weights \cite{c38}. \textit{Cosine Similarity (CosSim)}: This algorithm only considers auxiliary task signals when the corresponding gradient aligns with the main task gradient direction \cite{c15}. \textit{OL-AUX}: An online learning algorithm that uses auxiliary task gradient directions, computed on previous batches to update auxiliary task weights. The gradient directions are computed on past batches and updated with gradient descent every few steps \cite{c9}. \textit{PCGrad}: An algorithm that projects conflicting gradient directions of a task to the normal plane of any other task it negatively interferes with \cite{c13}. \textit{CAGrad}: Conflict-Averse Gradient descent that aims to simultaneously minimize the average loss and leverages the worst local improvement of individual tasks \cite{c30}. \textit{GradNorm}: Gradient normalization algorithm that automatically tunes gradient magnitudes to dynamically balance training in deep multi-task models \cite{c2}. Selection of the baseline algorithms was performed based on the width of adoption and current state-of-the-art performance in dynamic MTL.
\subsection{Results \label{results}}
In this subsection, we summarize the results obtained from the experiments introduced in the previous sections. 
\paragraph{Adaptive weighting.} In Figure~\ref{Totalweights}, the average task weight, calculated by summing over all sample weights, is depicted for all CIFAR tasks for the three CIFAR-10 set-ups: 0, 40, and 70 percent label flips. Interestingly, we observe how SLGrad shifts more weight to the auxiliary tasks when more labels are flipped. Specifically, the average total weight increases for seven out of nine auxiliary tasks when comparing the 0 to the 40 percent scenario. In an MTL context, this implies that the algorithm increasingly exploits auxiliary tasks to learn a meaningful featurization to correctly classify the main task. The data used to generate Figure~\ref{Totalweights} is provided as a plot in the Supplemental Material (section C). 
\begin{figure}[h!]
    \centering
    \includegraphics[width=\columnwidth]{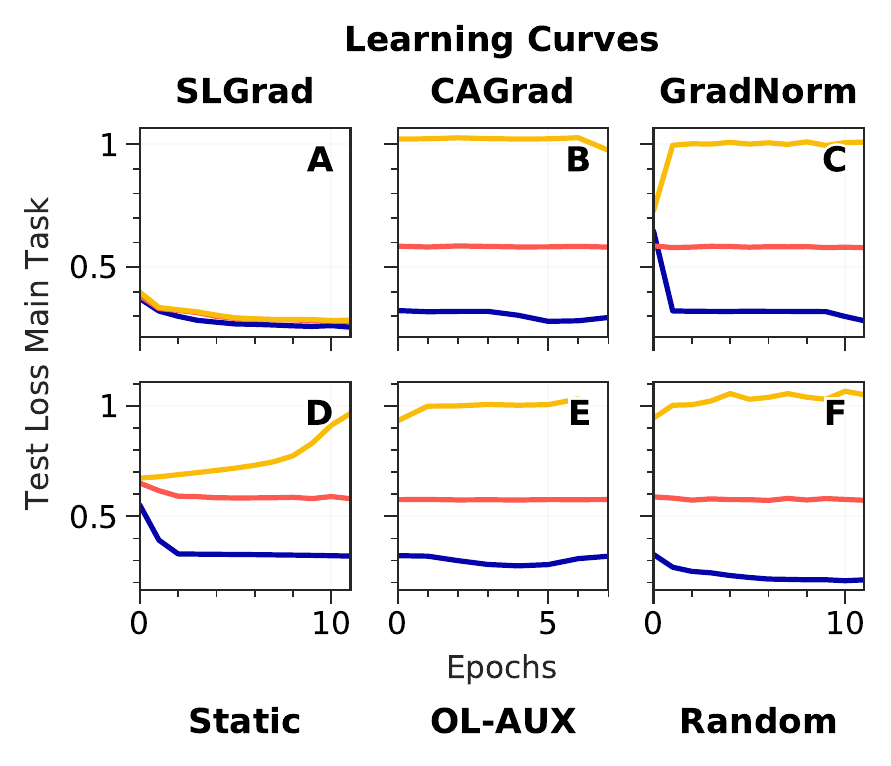}
    \caption{Learning curves for 0 (blue), 40 (red) and, 70 (yellow) percent of flipped labels for SLGrad and 5 other baseline algorithms on the semi-synthetic CIFAR 10 Uniform Flip dataset. The benchmark algorithms overfit the training set, while SLGrad effectively exploits the meta-objective signal.}
    \label{fig:learncurv}
\end{figure}
\paragraph{Generalization performance.}
As shown in Table~\ref{baselineRW}, SLGrad outperforms all the baseline algorithms in the flipped label set-ups. In the clean set-ups,  SLGrad also performs best for Multi-MNIST, and competitively for CIFAR-10.

Figure~\ref{fig:learncurv} depicts the learning curves for a selection of algorithms as measured on the test set. While all baseline algorithms overfit the inserted training noise, the learning trajectory of SLGrad remains able to generalize well. 

Finally, these results confirm that the benefit of sample-level weighting depends on the specific dataset and task at hand. This was previously only noted in a single-task context \cite{c999}.
\section{Conclusions and Future Work}
In this paper, we have investigated sample-level weighting as a solution to the auxiliary learning problem. To this end, SLGrad was proposed, a simple yet flexible MTL weighting algorithm. SLGrad performs dynamic effective weighting at the sample level and can filter out data set biases by exploiting a hold out meta-objective. From a MTL perspective, this paper can be seen as an extension towards more granular, i.e., sample-level task-specific weighting ($i$ in $w_{ij}$). From a sample weighting perspective, importance weighting is extended to multiple dimensions (i.e., tasks, or $j$ in $w_{ij}$). 
The resulting unified perspective offers multiple promising avenues for future research.\\
\indent Regarding applied research, depending on the meta-objective formulation, SLGrad can be applied across all domains where importance weighting used, such as causality, cost-sensitive learning, and learning with class imbalance. Additionally, SLGrad provides novel perspectives for practitioners and researchers alike to incorporate auxiliary tasks to solve the problem at hand.\\
\indent Regarding fundamental research, deeper understanding of the complex interactions on a granular level between auxiliary tasks and the main task across various types of network architectures and data modalities is needed. Moving beyond the task level marks a paramount step in this process.
\bibliography{SLGRAD}

\onecolumn
\renewcommand{\theequation}{S.\arabic{equation}}
\renewcommand{\thetheorem}{1}
\section*{Supplement A: Theoretical proofs}
The first supplement consists of the theoretical proofs of the two Theorems proposed in section \ref{theo} of the main paper. 
\section*{A1. Proof of Theorem 1 \label{proof1}}
\begin{theorem}
Assume $\mathcal{M}^{val}_{m}$ is a differentiable function representing the meta-(minimization)- objective as measured on the main task validation set, $\mathcal{L}_{SLGrad}$ (as defined in equation \ref{SLGradloss}) is the total training loss and $\mathcal{L}_{ij}(\theta^{(t)})\geq 0$ is the sample-level loss corresponding to sample $i$ of task $j$. If the sample-level weights $w_{ij}$ are computed by the SLGrad update rule (defined by equations \ref{SLGradloss} to \ref{weightup}), and the model parameter vector $\theta^{(t)}$ is updated through gradient descent (equation \ref{gradd}), then the following statement holds:
\begin{equation}
    \mathcal{M}^{val}_{m}(\theta^{(t+1)}) - \mathcal{M}^{val}_{m}(\theta^{(t)}) \leq 0
    \label{mono2}
\end{equation}
$\forall t \in [0, \#training steps]$.
\end{theorem}
The proof of Theorem \ref{theo1} consists of two lemma's, which we will proof first.\\
\begin{lemma}
Consider $N_B$ as the batch size and $N_{T}$ as the number of tasks. If the assumptions made in Theorem \ref{theo1} hold and if $\nabla_{\theta^{(t)}}\mathcal{L}_{ij}(\theta^{(t)}).\nabla_{\theta^{(t)}}\mathcal{M}_{m}^{val}(\theta^{(t)})=0$ $\forall i \in \{0, N{_B}-1\}$ and $\forall j \in \{0, N_{T}-1\}$ then:
\begin{equation}
   \mathcal{M}_{m}^{val}(\theta^{(t+1)})\approx \mathcal{M}_{m}^{val}(\theta^{(t)})
\end{equation}
holds $\forall t$.
\label{lem1}
\end{lemma}
\begin{proof}
As $\mathcal{M}^{val}_{m}$ is a differentiable, scalar-valued function, the first-order Taylor series expansion about a point $\theta^{(t)}$ is
\begin{align}
    \forall t: 
        &\mathcal{M}_{m}^{val}(\theta^{(t+1)}) \\&=\mathcal{M}_{m}^{val}(\theta^{(t)}-\eta  \nabla_{\theta^{(t)}}\mathcal{L}_{SLGrad}(\theta^{(t)}))\\
        &\approx \mathcal{M}_{m}^{val}(\theta^{(t)}) -\eta[\nabla_{\theta^{(t)}} \mathcal{M}_{m}^{val}(\theta^{(t)})]^{T}\nabla_{\theta^{(t)}}\mathcal{L}_{SLGrad}(\theta^{(t)}).
        \label{tay}
\end{align}
Using the SLGrad update rule (defined by equations \ref{SLGradloss} to \ref{weightup} of the main paper), we find that
    \begin{align}
        & \nabla_{\theta^{(t)}}\mathcal{L}_{SLGrad}(\theta^{(t)}) \\ &=\sum_{i}^{N_{B}}\sum_{j}^{N_{T}}\nabla_{\theta^{(t)}}w_{ij}\mathcal{L}_{ij}(\theta^{(t)})\\ &=
        \sum_{i}^{N_{B}}\sum_{j}^{N_{T}}\max \large[0,\nabla_{\theta^{(t)}}\mathcal{L}_{ij}(\theta^{(t)})\cdot\nabla_{\theta^{(t)}}\mathcal{M}_{m}\large]\nabla_{\theta^{(t)}}\mathcal{L}_{ij}(\theta^{(t)})
        \\&=0,
        \end{align}
where we used $\forall i,j: \nabla_{\theta^{(t)}}\mathcal{L}_{ij}(\theta^{(t)})\cdot\nabla_{\theta^{(t)}}\mathcal{M}_{m}=0$ in the last step. Inserting this into equation (\ref{tay}), we thus obtain
\begin{equation}
    \mathcal{M}_{m}^{val}(\theta^{(t+1)})-\mathcal{M}_{m}^{val}(\theta^{(t)}) \approx 0.
\end{equation}
\end{proof}
\begin{lemma}
Consider $N_B$ as the batch size and $N_{T}$ as the number of tasks. If the assumptions made in Theorem \ref{theo1} hold and if $\exists i \in \{0, N_{B}-1\}$ and $ j \in \{0, N_{T}-1\}$ such that
\begin{equation}
     \nabla_{\theta^{(t)}}\mathcal{L}_{ij}(\theta^{(t)})\cdot\nabla_{\theta^{(t)}}\mathcal{M}_{m}^{val}(\theta^{(t)})\neq 0,
\end{equation}
then: 
\begin{equation}
    \mathcal{M}_{m}^{val}(\theta^{(t+1)}) - \mathcal{M}_{m}^{val}(\theta^{(t)}) < 0.
\end{equation}
\label{lem2}
\end{lemma}
\begin{proof}
Approximate $\mathcal{M}_{m}^{val}(\theta^{(t+1)})$ by its first-order Taylor series expansion
\begin{equation}
   \mathcal{M}_{m}^{val}(\theta^{(t+1)})\approx \mathcal{M}_{m}^{val}(\theta^{(t)}) - \eta [\nabla_{\theta^{(t)}} \mathcal{M}_{m}^{val}(\theta^{(t)})]^{T}\nabla_{\theta^{(t)}}\mathcal{L}_{SLGrad}(\theta^{(t)}).
\end{equation}
 We thus need to prove that
\begin{equation}
\begin{split}
    \eta [\nabla_{\theta^{(t)}} \mathcal{M}_{m}^{val}(\theta^{(t)})]^{T}\nabla_{\theta^{(t)}}\mathcal{L}_{SLGrad}(\theta^{(t)}) & > 0.
\end{split}  
\end{equation}
As $\eta > 0$ and we can ignore contributions for which $w_{ij}=0$, this reduces to proving that the following sum is positive:
\begin{equation}
\begin{split}
   & \sum_{i: w_{ij} \neq 0}^{N_{B}}\sum_{j: w_{ij} \neq 0}^{N_{T}}  w_{ij} \nabla_{\theta^{(t)}}[\mathcal{M}_{m}^{val}(\theta^{(t)})]^{T}  \nabla_{\theta^{(t)}}\mathcal{L}_{ij}(\theta^{(t)})\\ &
       = \sum_{i: w_{ij} \neq 0}^{N_{B}}\sum_{j: w_{ij} \neq 0}^{N_{T}}  w_{ij}\nabla_{\theta^{(t)}}\mathcal{M}_{m}^{val}(\theta^{(t)})\cdot\nabla_{\theta^{(t)}}\mathcal{L}_{ij}(\theta^{(t)}).
    \label{imp}
    \end{split}
\end{equation}
 Next, we used the fact that $[\nabla_{\theta^{(t)}}\mathcal{M}_{m}^{val}(\theta^{(t)})]^{T}$ and $\nabla_{\theta^{(t)}}\mathcal{L}_{ij}(\theta^{(t)})$ are row and column vectors respectively, such that
\begin{equation}
\begin{split}
     [\nabla_{\theta^{(t)}}\mathcal{M}_{m}^{val}(\theta^{(t)})]^{T}   \nabla_{\theta^{(t)}}\mathcal{L}_{ij}(\theta^{(t)})&\\ =\nabla_{\theta^{(t)}}\mathcal{M}_{m}^{val}(\theta^{(t)})\cdot\nabla_{\theta^{(t)}}\mathcal{L}_{ij}(\theta^{(t)}).
\end{split}
\end{equation}
All the remaining $w_{ij}$ are now positive and non-zero by definition. Therefore, using that $w_{ij}=0$ if  $\nabla_{\theta^{(t)}}\mathcal{M}_{m}^{val}(\theta^{(t)})\cdot\nabla_{\theta^{(t)}}\mathcal{L}_{ij}(\theta^{(t)})\leq 0$, we know that for all the terms occurring in the sum in equation (\ref{imp}):
\begin{equation}
    \nabla_{\theta^{(t)}}\mathcal{M}_{m}^{val}(\theta^{(t)})\cdot\nabla_{\theta^{(t)}}\mathcal{L}_{ij}(\theta^{(t)})>0,
\end{equation}
Consequently, 
\begin{equation}
    \sum_{i: w_{ij} \neq 0}^{N_{B}}\sum_{j: w_{ij} \neq 0}^{N_{T}} \eta w_{ij}\nabla_{\theta^{(t)}}\mathcal{M}_{m}^{val}(\theta^{(t)})\cdot\nabla_{\theta^{(t)}}\mathcal{L}_{ij}(\theta^{(t)})>0,
\end{equation}
such that 
\begin{equation}
    \mathcal{M}_{m}^{val}(\theta^{(t+1)}) < \mathcal{M}_{m}^{val}(\theta^{(t)}), 
\end{equation}
which proves the Lemma. 
\end{proof}
We prove Theorem \ref{theo1} by combining Lemma's \ref{lem1} and \ref{lem2}. \\ \\
If we further assume, the generalization metric $\mathcal{M}_{m}^{val}$ is bounded from below, the non-decreasing behavior implies that the metric convergence to a metric due to the monotone convergence theorem. Additionally, if $\mathcal{M}_{m}^{val}$ is convex, this minimum corresponds to the global minimum. 

\section*{A2. Proof of Theorem 2}
\renewcommand{\thetheorem}{2}
The proof of Theorem \ref{taylortheo} is given below. 
\begin{theorem}
Assume $\mathcal{M}_{m}$ is a differentiable function, representing the validation metric on the main task and $\mathcal{L}^{(t)}_{SLGrad}$ the total training loss (as defined in equation \ref{SLGradloss} in the main paper). Consider $\theta^{(t)}$ as the model parameter vector at time $t$. If the sample-level weights $w_{ij}$ are computed by the SLGrad update rule (defined by equations \ref{SLGradloss} to \ref{weightup}), and the model parameter vector $\theta^{(t)}$ is updated through gradient descent (equation \ref{gradd}), then the gain in the validation metric, can be approximated as
\label{taylortheo}
\begin{equation}
     \mathcal{M}_{m}^{val}(\theta^{(t+1)}) - \mathcal{M}_{m}^{val}(\theta^{(t)}) \approx - \eta[\nabla_{\theta^{(t)}}\mathcal{M}^{val}_{m}]^{T}\nabla_{\theta^{(t)}}\mathcal{L}_{SLGrad} ,
     \label{tayl}
\end{equation}
where $\eta$ is the learning rate. 
\end{theorem}
\begin{proof}
$\forall t \in [0, \#trainingsteps]$ we can expand the main task metric $\mathcal{M}_{m}(\theta^{(t+1)}$) as
\begin{align}
   \mathcal{M}_{m}^{val}(\theta^{(t+1)})&\approx \mathcal{M}_{m}^{val}(\theta^{(t)}-\eta \nabla_{\theta^{(t)}}\mathcal{L}_{SLGrad}) \\ & \approx
   \mathcal{M}_{m}^{val}(\theta^{(t)}) - \eta [\nabla_{\theta^{(t)}} \mathcal{M}_{m}^{val}(\theta^{(t)})]^{T}\nabla_{\theta^{(t)}}\mathcal{L}(\theta^{(t)}),
\end{align}
where we used the first-order Taylor approximation of a differentiable, scalar-valued function about a point $\theta^{(t)}$.
Therefore, the gain in the validation metric, resulting from the parameter update  $\theta^{(t)} \to \theta^{(t+1)}$,  can be approximated as
\begin{align}
    \mathcal{M}_{m}^{val}(\theta^{(t+1)}) -  \mathcal{M}_{m}^{val}(\theta^{(t)}) &  \approx - \eta [\nabla_{\theta^{(t)}}\mathcal{M}_{m}]^{T}\nabla_{\theta^{(t)}}\mathcal{L}_{SLGrad}.
    \label{Taylorapp}
\end{align}
\end{proof}
Experimental justification for the use of the Taylor approximation is provided in section C1. 
\newpage
\section*{Supplement B: Implementation details \label{exp}}
The implementation details comprising the data, the baseline neural network and the hyperparameter configuration are discussed for all experimental setups. Accordingly, this section is organized as follows; the simulation study is discussed in subsection B1, CIFAR-10 in subsection B2, and Multi-MNIST in subsection B3.
Pytorch was used for the implementation of all the experiments. For the implementation of CAGrad and PCGrad, we adapted code from LibMTL \cite{LibMTL}. Note that we logged all the metrics, losses, and hyperparameter configurations using Weights $\&$ Biases \cite{wandb}. The corresponding projects will be made available.
\section*{B1. Simulation Study \label{sim}}
This subsection contains detailed information about the simulation study presented in section 4 of the paper. Concretely, we elaborate on how the toy regression tasks were generated in subsection B1.1. The neural network used as a backbone for the experiments and the corresponding hyperparameter configurations are presented in subsections B1.2 and B1.3 respectively.
\subsection*{B1.1 Data \label{simdat}}
The toy multi-task regression set-up is generated from the following function:
\begin{equation}
    y_{out}=f_{i}(x_{in})= \sigma_{i}\tanh [(\textbf{B}+\epsilon_{i})\textbf{x}],
    \label{basis}
\end{equation}
where $x_{in}$ is the ten-dimensional input vector. As such, we followed the approach of \cite{c2} and \cite{hyda}. The validation and test sets contain 200 samples while 1000 samples are used to train the network. We generated one auxiliary task with the same function as the main task (defined by equation (\ref{basis})) and made the $\epsilon_{i}$ task-dependent.  The common basis $\textbf{B}$ and task-dependent $\epsilon_{i}$ represent constant matrices generated IDD from $\mathcal{N}(0,\sqrt{1})$ and $\mathcal{N}(0,\sqrt{3.5})$ respectively. As the task distributions have a common basis, it is expected that the auxiliary task to be generally helpful to the main task \cite{c2}. In order to test the robustness of the proposed algorithm, different amounts (0, 40, and 70 percent) of Gaussian noise were added to the regression task in equation (\ref{basis}):
\begin{equation}
    y_{out}=y_{out}+\mathcal{N}(0,\sqrt{2})
\end{equation}
Note that the validation set remains unaltered. By labeling each observation as (non)-noisy, we kept track of how weights are distributed over the observations. As such, weight distributions as presented on Figure 2 in the main paper were created. 
\subsection*{B1.2 Model Architecture \label{NNSim}}
The same model architecture is used for all the different toy setups and baseline algorithms: a standard multi-task neural network with two to four shared layers (64 neurons each) and one to four task-specific layers (32 neurons each) for the two tasks. 
\subsection*{B1.3 Hyperparameters \label{Hypsim}}
A grid-search was performed to optimize the hyper-parameters (learning rate (LR), batch size (BS), number of shared (SL) and task-specific (TL) layers ). For completeness, the grid-search was performed independently for each algorithm. The corresponding optimal hyperparameter configurations for each algorithm are presented in Table \ref{toyhyp} below. The parameters are optimized for one and the same initialization after which the optimal configuration is applied to the additional initializations. The values from Table 2 in the main paper were obtained by averaging the best performance for three different initializations (random seeds) for each algorithm.  
\begin{table}[h!]
\centering
\begin{tabular}{l l l l l }
\\ \textbf{ALGORITHM} & \textbf{LR} & \textbf{BS} & \textbf{SL}&  \textbf{TL} \\ \hline \\
 OL-AUX & 0.1 &  64 & 2  & 2\\ 
 PC Grad & 0.1 & 32 & 3& 3\\ 
 CA Grad & 0.1 & 64 & 2 & 2 \\ 
 Gcosim & 0.01  & 64 &  3 & 4 \\ 
 Static & 0.01 & 32 & 4 & 4\\ 
 GradNorm &0.1 & 32 & 2& 2\\
 Random & 0.01 & 64 & 3 & 4\\  
SLGrad (ours) & 0.1 & 32 & 3 & 4 \\ 
\end{tabular}
\caption{Optimal hyperparameters for toy dataset}
\label{toyhyp}
\end{table}
\section*{B2. CIFAR-10 \label{cif}} 
Implementation details corresponding to the experiments performed on the CIFAR-10 dataset are discussed in this subsection. First, in section B2.1, the standard dataset and its transformations are presented. Then, the neural network used as a backbone for the experiments and the corresponding hyperparameter configurations are provided in subsections B2.2 and B2.3 respectively.
\subsection*{B2.1 Data}
\indent CIFAR-10 \cite{cifar10} is a benchmark dataset containing 60000 color images (50000 train and 10000 test images) of size 32x32 that each belong to one of ten classes. For the experiments 20000 images were subsampled from the original training dataset for training. Additionally, the multi-label classification task was transformed into 10 binary classification tasks. One of these classification tasks (i.e., predicting if an airplane appears on an image or not) was chosen to become the main task of interest and the other nine were used as auxiliary tasks. 
Next to the standard \say{clean} CIFAR-10 multi-task setup, we also generated two semi-synthetic setups, following \cite{c35}:
\begin{itemize}
    \item \textit{Uniform Label Flips}: samples from all classes are assigned a different label with 40\% and 70\% probability, respectively. 
    \item \textit{Background Label Flips}: all flipped labels (20 percent) are flipped to the same class, i.e., the background class. This type of flip also induces class imbalance as the background class will be likely to dominate. 
\end{itemize}
Note that, in creating the semi-synthetic setups the validation data set (20 percent of the training set) is clean for all algorithms. 
\subsection*{B2.2 Model Architecture}
We employ the LeNet-5 architecture \cite{cLENET} as the backbone for the experiments on CIFAR-10. All the layers were used as a shared encoder (except the fully connected layer). On top of this shared encoder, one to four fully connected layers were added as task-specific functions with ReLU activation functions. The Binary Cross Entropy with a sigmoid is used as the task-specific loss function for all tasks.
\subsection*{B2.3 Hyperparameters}
For completeness, an independent grid-search (learning rate (LR $\in \{ 0.1, 0.01, 0.001\}$), number of task-specific layers (TL: 1 to 4) and batch size (BS $\in \{32, 64, 128 \}$) was performed for each dynamic weighting algorithm separately. The optimal hyperparameters for each algorithm are presented in Table 2 below. Note that the optimal hyperparameters for one initialization are used on three additional initializations. The values in Table 3 in the main paper are obtained by averaging the test performance over these three different initializations (random seeds) and flipped label experiments.
\begin{table}[h!]
\centering
\begin{tabular}{l l l l l }
\\ \textbf{ALGORITHM} & \textbf{LR} & \textbf{BS} & \textbf{TL} \\ \hline \\
 OL-AUX & 0.001 & 64  & 2  &\\ 
 PC Grad & 0.01 & 64  & 2 &\\ 
 CA Grad & 0.01 & 64 & 2 &\\ 
 GCosim &  0.001 & 128 & 3  & \\ 
 GradNorm & 0.1 & 128 & 1 & \\
 Static & 0.1  & 32 & 3 &\\ 
 Random & 0.1 & 32  & 2 &\\  
SLGrad (ours) & 0.1  & 128  & 2 \\ 
\end{tabular}
\caption{Hyperparameter optimization on CIFAR-10 clean dataset.}
\label{cifoptim}
\end{table}
\newpage
\section*{B3. Multi-MNIST \label{mnist}}
This subsection contains detailed information about the experiments performed on the Multi-MNIST dataset. To start with, subsection B3.1 provides more information about the data. Thereafter, the neural network and its corresponding hyperparameter configurations are discussed in subsections B3.2 and B3.3 respectively.
\subsection*{B3.1 Data} Multi-MNIST \cite{c1000} is a MTL version of the original MNIST dataset \cite{cMNIST} generated by \cite{c1000} and commonly used for benchmarking in the MTL literature. \cite{c1000} generated the data by overlaying digits from the MNIST dataset on top of other digits of the same set (training or test) but from another class. As the digits are shifted up to 4 pixels in each direction, the dataset contains 36x36 pixel images. We use this dataset in the MTL setup by defining the classification of the left and right digits on each image as different tasks. The dataset contains 60k training and 20k test images but only 20k training and 5k test samples are used. The first task (identifying the left digit) is identified as the main task for the experiments while the second task (identifying the right digit) is used as an auxiliary task. \\
\subsection*{B3.2 Model Architecture}
The same adaptation of the LeNet-5 architecture as for the experiments on CIFAR-10 is used. Note that the binary cross entropy with the sigmoid function was replaced by the cross entropy function with softmax. For the implementation, we adapted code from \cite{cmultiop}.
\subsection*{B3.3 Hyperparameters}
To optimize each algorithm, a grid-search over the hyperparameters: learning rate (LR $\in \{0.1, 0.01, 0.001, 0.0001 \}$) and batch size (BS $\in  \{32, 64, 128 \}$) was performed for each algorithm separately. The corresponding optimal hyperparameters for one initialization (random seed) are presented in Table \ref{MNIST voor}. The results in Table 4 in the paper are obtained by averaging the test performance for three different initializations.
\begin{table}[h!]
\centering
\begin{tabular}{l l l }
\\ \textbf{ALGORITHM} & \textbf{LR} & \textbf{BS}  \\ \hline \\
 OL-AUX & 0.001 &  128   \\ 
 PC Grad &  0.1 & 128 \\ 
 CA Grad & 0.001 & 32  \\ 
 Gcosim &  0.001&  128 \\ 
 Static & 0.001 & 128  \\ 
 GradNorm & 0.0001 & 128  \\
 Random & 0.001 & 32 \\  
SLGrad (ours) &  0.001 & 128 \\ 
\end{tabular}
\caption{Hyperparameter optimization on the Multi-MNIST dataset.}
\label{MNIST voor}
\end{table}
\newpage
\section*{Supplement C: Additional Results}
In this section, additional experimental results are presented. In subsection C1, the use of the Taylor approximation (equation \ref{tay}) is justified by experimental observations. Next, additional plots supporting the results presented in sections 4 and 5 of the paper are provided.
\subsection*{C1. Taylor Approximation}
We compare the learning curve of SLGrad when the algorithm is applied with and without the use of Taylor approximation (\ref{Taylorapp}).  Applying the approximation consists of not explicitly computing the look-ahead update (as discussed in section 3.2 of the main paper and in \cite{c31,TAG}) to obtain $\mathcal{M}(\theta_{t+1})$, but instead using only $\mathcal{M}(\theta_{t})$. In essence, this results in the removal of one forward and parameter update on each training pass. Figure \ref{toytay} provides the comparison of the learning curves on the experiments with and without approximation. We see that the learning curves remain seemingly unaltered for the toy regression task (left panel) and similar for CIFAR-10 (right panel) where a small improvement is still observed for the explicit computation. 
\begin{figure}[h!]
    \centering
    \includegraphics[scale=0.7]{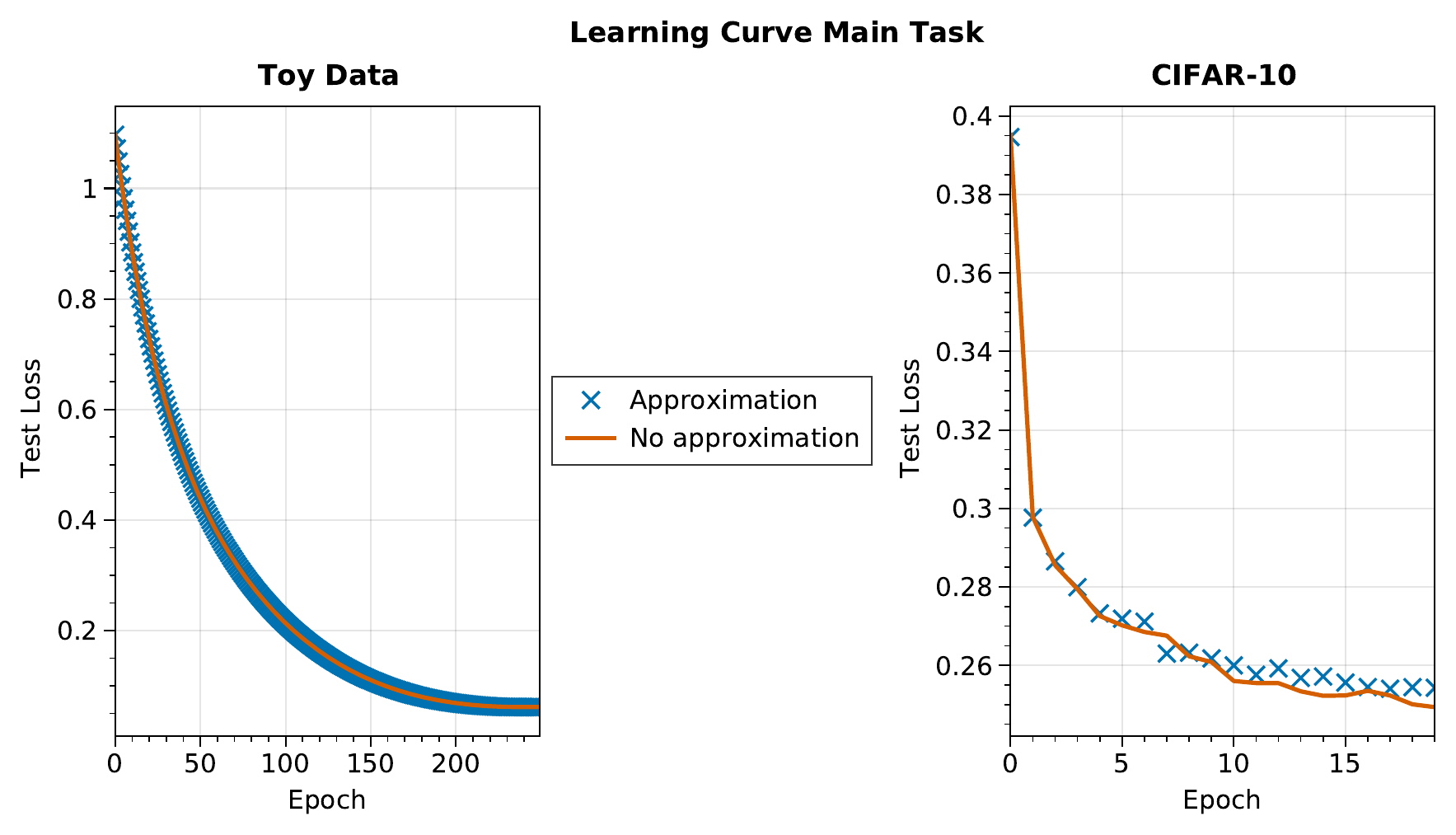}
    \caption{The comparison between the learning curve of SLGrad with and without the use of the Taylor approximation for the Toy simulation (A) and CIFAR-10 experiments (B).}
    \label{toytay}
\end{figure}
\subsection*{C2. Additional Experimental Results}
Additional experimental support for the results discussed in sections 4 and 5 of the main paper is presented in this section. Subsection C2.1 contains additional plots to support the weight distributions provided in section 4 of the paper. Next, subsection C2.2 contain additional plots supporting section 5 in the main paper.
\newpage
\subsubsection*{C2.1 Sample Weighting}
Figure \ref{totweight_aux} shows the evolution of the total task weight of the main and auxiliary tasks for the experiments on the clean toy data. The total task weights are obtained by summing over the sample level weights for each task distribution separately. At the beginning of the learning procedure, most of the weight is given to the main task, after which it decreases at the end of the learning procedure. 
\begin{figure}[h!]
    \centering
    \includegraphics[scale=0.9]{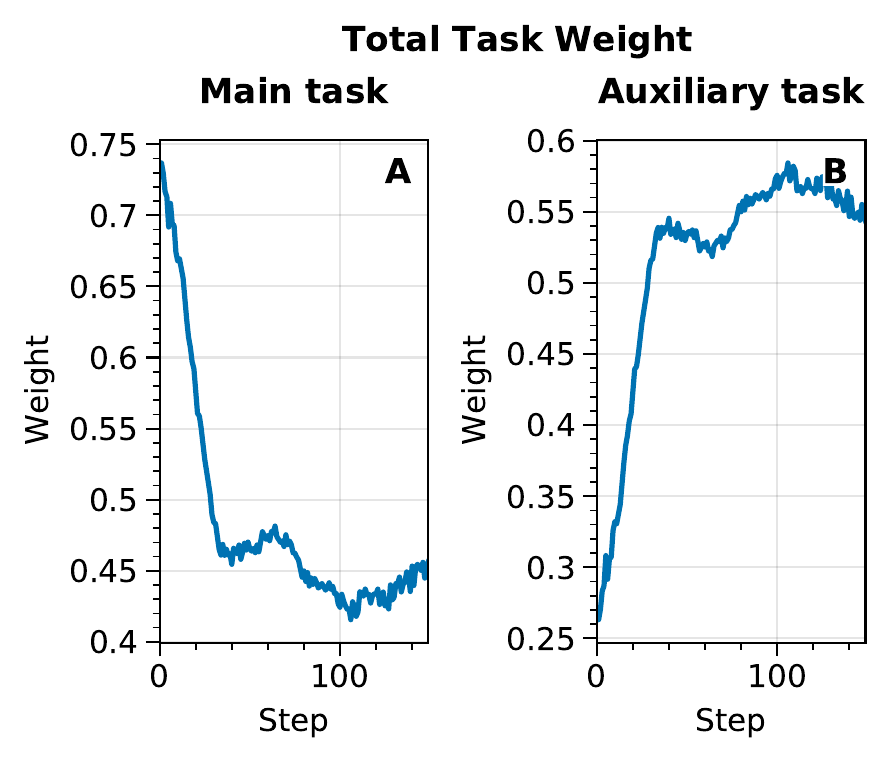}
    \caption{Evolution of the total task weight of the main task (A) and auxiliary task (B) for a clean toy setup. }
    \label{totweight_aux}
\end{figure}
\subsubsection*{C2.2 Task Weighting}
The total task weight of the main task assigned by SLGrad when applied on Cifar-10 for different amounts of flipped labels (0, 40, 70) is shown in Figure \ref{totweigl}. This plot contains the data used to generate Figure 3 in the main paper. As discussed in section 5 of the main paper, SLGrad decreases the weight assigned to the main task as the amount of flipped labels in the main task distribution increases. 
\begin{figure}[h!]
    \centering
    \includegraphics[scale=0.9]{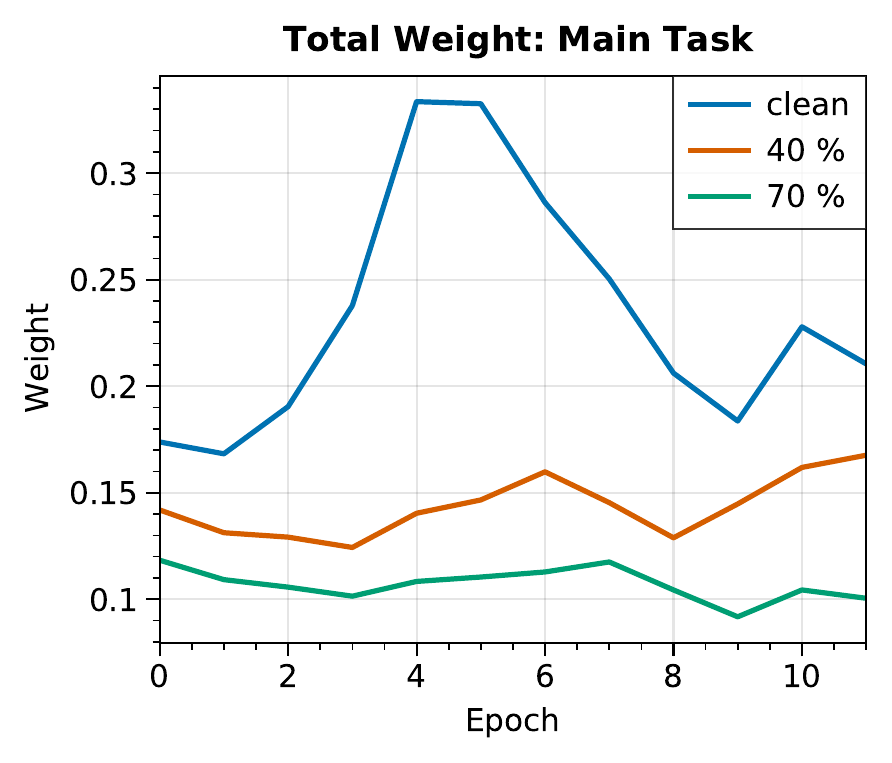}
    \caption{Comparison of total task weight of the main task for SLGrad applied on CIFAR-10 for different amounts of flipped labels (0, 40, 70 percent). This plot was used to generate Figure 3 in the paper. The hyperparameters used for the three experiments are equal.}
    \label{totweigl}
\end{figure}

\end{document}